%% file: emamjomeh-zadeh21.tex
\documentclass[final,12pt]{alt2021} 

\usepackage[utf8]{inputenc} 
\usepackage[T1]{fontenc}    
\usepackage{url}            
\usepackage{booktabs}       
\usepackage{amsfonts}       
\usepackage{nicefrac}       
\usepackage{microtype}      

\input{commands}

\title[Approximate Regret Bounds for Sleeping Experts/Bandits]{Adversarial Online Learning with Changing Action Sets: \\ Efficient Algorithms with Approximate Regret Bounds}

\usepackage{times}

\altauthor{%
 \Name{Ehsan Emamjomeh-Zadeh} \Email{emamjome@usc.edu}\\
 \Name{Chen-Yu Wei} \Email{chenyu.wei@usc.edu}\\
 \Name{Haipeng Luo} \Email{haipengl@usc.edu}\\
 \Name{David Kempe} \Email{david.m.kempe@gmail.com}\\
 \addr University of Southern California
}

\begin{document}

\maketitle

\begin{abstract}
\input{abstract}
\end{abstract}

\section{Introduction}
\label{sec:introduction}
\input{introduction}

\section{Problem Setting and Preliminaries}
\label{sec:problem-setting}
\input{setting}

\section{The Full-information Setting}
\label{sec:full-information}
\input{full-information}

\section{The Bandit Setting}
\label{sec:bandit}
\input{bandit}

\section{Conclusions}
\label{sec:conclusions}
\input{conclusions}


\section*{Acknowledgement}
We thank Elad Hazan and He Jiang for working with us in the early stage of this project, and thank anonymous reviewers for providing very constructive comments. EE and DK were supported
in part by grants NSF IIS-1619458 and ARO W911NF1810208. HL and CYW were supported in part by NSF IIS1755781 and NSF IIS1943607.

\bibliography{davids-bibliography/names,davids-bibliography/conferences,davids-bibliography/bibliography,davids-bibliography/publications,local}



\end{document}

%% file: commands.tex
\usepackage{xspace}
\usepackage{amsmath}

\usepackage{makecell}
\usepackage{bbm}
\usepackage{mathrsfs}
\usepackage{amssymb}
\usepackage{bm}

\usepackage{tabulary}
\usepackage{graphicx}
\usepackage{wrapfig}
\usepackage{algorithm, algpseudocode}

\DeclareMathOperator*{\argmin}{argmin} 

\newcommand{\vc}[1]{\ensuremath{\bm{#1}}\xspace}
\newcommand{\e}{\mathrm{e}\xspace}
\newcommand{\hedgetree}{\textsc{HATT}\xspace}
\newcommand{\hedgevote}{\textsc{HOPP}\xspace}
\newcommand{\hedge}{\textsc{Hedge}\xspace}
\newcommand{\updatehedge}{\textsc{EWU}\xspace}
\newcommand{\calS}{\mathcal{S}}
\newcommand{\calT}{\mathcal{T}}
\newcommand{\calA}{\mathcal{A}}
\newcommand{\calE}{\mathcal{E}}
\newcommand{\calZ}{\mathcal{Z}}
\newcommand{\calH}{\mathcal{H}}
\newcommand{\tournament}{\textsc{Tournament}\xspace}
\newcommand{\one}{\mathbbm{1}}

\newcommand{\E}{\mathbb{E}}
 
\newcommand{\vote}{\textsc{SelectionRule}\xspace}
\newcommand{\level}{\textsc{Level}\xspace}
\newcommand{\lev}{\text{level}}
\newcommand{\Reg}{\text{Reg}\xspace}

\newcommand{\rt}{\text{root}}
\newcommand{\winner}{\text{winner}}

\newcommand{\order}{\mathcal{O}}

\newtheorem{property}{Property}

\def\QED{{\phantom{x}} \hfill \ensuremath{\rule{1.3ex}{1.3ex}}}

%% file: abstract.tex
We revisit the problem of online learning with sleeping experts/bandits: in each time step, only a subset of the actions are available for the algorithm to choose from (and learn about).
The work of \citet{kleinberg2010regret} showed that there exist no-regret algorithms which perform no worse than the \emph{best ranking of actions} asymptotically. Unfortunately, achieving this regret bound appears computationally hard: \citet{kanade2014learning} showed that achieving this no-regret performance is at least as hard as PAC-learning DNFs, a notoriously difficult problem.

In the present work, we relax the original problem and study computationally efficient \emph{no-approximate-regret} algorithms: such algorithms may exceed the optimal cost by a multiplicative constant \emph{in addition to} the additive regret. We give an algorithm that provides a no-approximate-regret guarantee for the general sleeping expert/bandit problems. For several canonical special cases of the problem, we give algorithms with significantly better approximation ratios; these algorithms also illustrate different techniques for achieving no-approximate-regret guarantees. 

%% file: introduction.tex

Online learning with a fixed set of actions is a well-studied problem: the learner sequentially selects one of $N$ actions and receives some feedback on the actions' losses.
In the full-information setting (i.e., the \emph{expert} problem~\citep{FreundSc97}), the feedback is the losses of all actions, while in the bandit setting (i.e., the multi-armed bandit problem~\citep{auer2002nonstochastic}), the feedback is only the loss of the chosen action.
In either case, the learner's goal is to minimize her regret over $T$ rounds, defined as the difference between her total loss and the loss of the best fixed action.
It is well-known that efficient algorithms exist with sublinear regret of order $\order(\sqrt{T})$ (known as ``no-regret'' algorithms).

In many situations, however, not every action is available every time.
Take horse racing as an example, where each action corresponds to betting on a horse.
While there is a fixed set of horses each season, only a small subset of them are competing in any one race; thus, only a subset of actions are available to choose from.
Other examples include recommendation systems where products are not available at all times; this is particularly relevant for food (where each action corresponds to a
restaurant or a special meal) or news (where each action corresponds to a category of news).

To capture these situations, a model called sleeping expert/bandit has been proposed~\citep{freund1997using}, where in each round, only some actions are \emph{awake} (i.e., available to be chosen and learned about), while the others are \emph{asleep}.
The standard regret measure no longer makes sense; in particular, there might not even be a fixed action which is awake all the time.
\citet{freund1997using} proposed to measure regret against a particular action only for the rounds when this action is awake.
As an alternative, \citet{kleinberg2010regret} proposed to measure regret against the \emph{best ranking} of the $N$ actions, which naturally selects the available action with the highest ranking in each round.
This latter performance measure is especially suited for applications such as horse racing or recommendation systems, and is the focus of our work.

In this setup, \citet{kleinberg2010regret} proposed algorithms with optimal regret $\order(\sqrt{NT\log(N)})$ for full-information feedback and $\order(\sqrt{NKT\log(N)})$ for bandit feedback; here, $K$ is an upper bound on the number of available actions in each round. \citet{kleinberg2010regret} made no assumptions at all about how the available sets and actions' losses are chosen; i.e., their results hold in the adversarial setting.
Unfortunately, their algorithms are computationally inefficient --- they require maintaining information about all $N!$ rankings explicitly. 
On the other extreme, a trivial algorithm that treats each possible available subset independently achieves regret that is exponential in $K$.

The computational inefficiency of these algorithms is no accident.
It was showed in~\citep{kanade2014learning} that achieving no-regret performance for this problem is at least as hard as PAC-learning DNFs, a notoriously difficult problem.
Follow-up work thus focused on developing efficient no-regret algorithms under additional assumptions, such as imposing distributional assumptions (see \textbf{Related work} below).

\renewcommand{\arraystretch}{1.2}
\begin{table*}[t]
   \centering
   \caption{Summary of main results. $N$ is the total number of actions, $K$ is an upper bound on the number of available actions at each round, $T$ is the number of rounds, $Z$ is the number of actions with zero loss at each round. Except for the first algorithm, all others can be implemented efficiently.}
   \label{tab:results}
\resizebox{\textwidth}{!}{%
   \begin{tabular}{|c|c|c|c|c|}
   \hline
  Algorithm &  Approx. Ratio ($\alpha$)      & $\alpha$-approx. Regret       & Feedback & Constraint  \\
  \hline   
  \citet{kleinberg2010regret}  &  $1$ &  $\order(\sqrt{NKT\log N})$     &      bandit &   inefficient  \\
  \hline   
  \makecell{Treat different  $\calA_t$ \\ independently  (Section \ref{sec:problem-setting})} &     $1$      &         $\sqrt{N^K K T}$        & bandit           &    \\
  \hline  
  \hedgetree (Section \ref{subsec:hedgetree})        &    $\order(\log K)$      &    $\order(N^2)$          & full-info           &  $Z=1$ \\
  \hline   
  \hedgevote (Section \ref{subsec:hedgevote})         & $\order(K^2)$          & $\order(N^4)$      &  full-info          &$Z=2$  \\
  \hline  
  {Bandit-\hedgetree}  (Section \ref{sec:bandit})        & $\order(\log K)$         & $\order(N\sqrt{KT} + N^2K$       )  & bandit       &  $Z=1$ \\
  \hline   
  \level (Section \ref{sec:bandit} )    & $N$      &  $N^2$     &  bandit    &   \\
  \hline   
  \end{tabular}
  }
\end{table*}

In this paper, we take a different approach to get around the computational hardness:
we still consider completely adversarial environments, but measure the learner's performance by \emph{$\alpha$-approximate regret} (for some approximation ratio $\alpha > 1$), which compares the learner's total loss to $\alpha$ times that of the best ranking.
Such approximate regret measures have been studied in other online learning problems, such as~\citep{garber2017efficient, roughgarden2018optimal}, but to our knowledge, our work is the first to consider them for the sleeping experts/bandits problem.
Our most general algorithm is a simple and efficient algorithm with approximation ratio $\alpha = N$ and regret $\order(N^2)$ (independent of $T$), even under bandit feedback (Section~\ref{sec:bandit}).

We also consider two cases with special structures in losses and develop different algorithms with much better approximation ratios (see Table~\ref{tab:results} for a summary).
First, for the case when in each round, there is only one zero-loss action, we improve the approximation ratio to $\log(K)$, both under full-information feedback (Section~\ref{subsec:hedgetree}) and bandit feedback (Section~\ref{sec:bandit}; in the latter case, the regret becomes $\order(\sqrt{T})$).
Note that the ``one zero-loss action'' structure is very common --- in the horse racing example, there is only one winner in each race,\footnote{%
In this example, in addition to the loss of betting on each horse, the bettor observes the ranking of each race as well. However, in the adversarial setting that we consider, this extra information is not useful since the ranking can be arbitrary from round to round.} 
and in a multi-class classification problem, only one class is the correct one.
Our algorithm is based on a novel way of aggregating several instances of the classic \hedge algorithm~\citep{FreundSc97} over action pairs via a tournament.

One might wonder whether in this restricted setting, the aforementioned computational hardness still applies.
Indeed, we generalize the argument of~\citep{kanade2014learning} and confirm that, even for this simple special case, obtaining no-regret algorithms is computationally hard (Theorem~\ref{thm:hardness}).

Next, we consider the case with two zero-loss actions in each round (e.g., betting on the winner or the runner-up has zero loss), and develop an algorithm in the full-information setting with approximation ratio $\order(K^2)$ and regret $\order(N^4)$ (Section~\ref{subsec:hedgevote}).
While the algorithm is also based on aggregating \hedge instances, it is significantly more complex and requires hedging over \emph{pairs of pairs} as well as \emph {triples}.
Our results shed light on how to deal with a small number of zero-cost actions, which is a common situation for machine learning problems with sparse rewards. 
Indeed, sparse rewards are studied in several recent works in the easier setting of fixed action sets (e.g.,~\citep{kwon2016gains, bubeck2018sparsity}).

\paragraph{Related work}
Several works propose efficient algorithms with exact regret (i.e., $\alpha=1$) guarantees under additional assumptions.
The original work of~\citet{kleinberg2010regret} considers a setting where the losses follow a fixed distribution, while~\citet{kanade2009sleeping}, \citet{neu2014online}, and \citet{saha2020improved} consider a setting where the action availability follows a fixed distribution.
\citet{hazan2012near} study the case when $K=2$ and achieve nearly optimal regret.
Recently, \citet{shayestehmanesh2019dying} studied a special case in which actions never wake up after falling asleep.



%% file: setting.tex

We consider the problem of online learning with a changing action set, also called the sleeping expert/bandit problem. Similar to the standard expert/bandit setting, the learner is faced with a set of actions $[N]= \{1, \ldots, N\}$. However, in each round $t$, only a subset $\calA_t\subseteq [N]$ is \emph{available}, and the learner can only choose actions from $\calA_t$ in that round. More precisely, the protocol is as follows. 
For each round $t=1, \ldots, T$, the adversary first chooses $\calA_t\subseteq [N]$ and $\ell_t(a)\in\{0,1\}$ for all $a\in\calA_t$, with $\calA_t$ revealed to the learner.
Then, the learner chooses an action $a\in\calA_t$, suffers loss $\ell_t(a_t)$, and receives some feedback.
%
%
%
%
%
We consider two different settings with different feedback:
     1) in the \emph{full-information} setting, the feedback is $(\ell_t(a))_{a\in\calA_t}$, i.e., the losses of all actions in $\calA_t$;
     2) in the \emph{bandit} setting, the feedback is $\ell_t(a_t)$, i.e., the loss of the chosen action.

Both $\calA_t$ and $\ell_t$ are decided by the adversary without any distributional assumptions.
We assume that the losses are binary, i.e., $\ell_t(a)\in \{0,1\}$. 
The goal of the learner is to be competitive with the best \emph{ranking} of actions. A ranking $\sigma$ specifies a total order on $[N]$, which is given by a bijection $m_{\sigma}: [N] \to [N]$, giving the position in the ranking for each element in $[N]$.
Due to frequency of use in our paper, we reserve the letter $\sigma$ itself for the mapping $\sigma : 2^{[N]} \setminus \emptyset \to [N]$
defined by $\sigma(\calS) = \argmin_{x \in \calS} m_{\sigma}(x)$.
That is, $\sigma(\calS)$ is the highest-ranked element of $\calS$, according to
$m_{\sigma}$. We write $\sigma(\{i,j\})$ or $\sigma(\{i,j,k\})$ as $\sigma(i,j)$ or $\sigma(i,j,k)$ for simplicity.

For a fixed ranking $\sigma$, we define its choice at time $t$ as its highest-ranked action among $\calA_t$ --- using our notation, this can be written as $\sigma(\calA_t)$. 
One standard way to measure the performance of the learner is to compare her total loss with that of the best ranking, formally defined as the \emph{regret}:
$\Reg_T = \sum_{t=1}^T \ell_t(a_t) - L^*$, 
where $L^*= \min_{\sigma} \sum_{t=1}^T \ell_t(\sigma(\calA_t))$ is the total loss of the best ranking. 
An algorithm with regret sublinear in $T$ performs almost as well as the best ranking in the long run.

Unfortunately, it was shown that achieving sublinear regret is computationally at least as hard as PAC-learning DNFs, for which no polynomial-time (in $N$) algorithm is known~\citep{kanade2014learning}. Therefore, we pursue the relaxed goal of providing polynomial-time algorithms that guarantee sub-linear \emph{$\alpha$-approximate} regret, defined as follows:
$\Reg_T^\alpha = \sum_{t=1}^T \ell_t(a_t) - \alpha L^*$.
Phrased in another way, our results can all be written as 
$\sum_{t=1}^T \ell_t(a_t)\leq \alpha L^* + \beta(T)$
for some $\beta(T)$ which grows sub-linearly in $T$;
our goal is to make $\alpha$ and $\beta(T)$ as small as possible.

Some of our guarantees depend on the (largest) cardinality of the sets $\calA_t$ of available actions, denoted by $K$. 
Note that achieving $\alpha=1$ and $\beta(T) = \order\left(\sqrt{\left(\sum_{i=1}^K \binom{N}{i}\right)KT }\right) = \order(\sqrt{N^K KT})$ efficiently is trivial.
One simply treats each possible $\calA_t$ 
as an independent problem with a fixed action set and runs a separate standard bandit algorithm with regret $\order(\sqrt{KT})$, then combines all regret bounds with a Cauchy-Schwarz inequality.\footnote{For more details, see~\citep[Lemma~3]{abernethy2010can} for the case with full-information and $K=2$.}
In contrast, our bounds are all polynomial in $N$ and $K$.

\paragraph{Notation.}
We use $\one[\calE]$ as an indicator function, which is $1$ if the event $\calE$ is true and $0$ otherwise. We use $\Delta_\calS$ to denote the probability simplex over a set $\calS$, i.e., $\Delta_\calS = \left\{ \vc{p}: \calS\to [0,1] ~\big\vert~ \sum_{a\in\calS}p(a)=1 \right\}$. 

\subsection{Preliminaries: The Hedge Algorithm}
\label{subsec:preliminary}
Most of our algorithms are based on the classic \hedge algorithm for the expert learning problem~\citep{FreundSc97}, which we review below. The setting of the expert learning problem is the same as our problem with full-information feedback, except the action set $\calA_t=\calS$ is fixed throughout. The \hedge algorithm is given as Algorithm~\ref{alg:hedge} --- we use $\calS$ instead of $\calA$ for its (fixed) action set and $[0,R]$ for the loss range, because we will later invoke it with different choices of $\calS$ and $R$.

\begin{figure}[t]
\begin{algorithm}[H]
\caption{\hedge $(\text{parameter: $\eta$})$}
\label{alg:hedge}
\DontPrintSemicolon
\KwIn{$\calS$}

%
\lForAll{$a \in \calS$}{$p_{1}(a)=\frac{1}{|\calS|}$.}

\For{$t=1, 2, \ldots, T$}{
  Sample $a_t\sim \vc{p}_t$.
  
  Receive $\ell_{t}(a)\in[0, R]$ for all $a\in\calS$.
%

   Let $\vc{p}_{t+1} = \updatehedge(\vc{p}_t, \vc{\ell}_t)$. 
}
\end{algorithm}
\begin{algorithm}[H]
\caption{\updatehedge (Exponential Weight Update)}
\DontPrintSemicolon
\KwIn{$\vc{p}_t\in\Delta_\calS, \vc{\ell}_t\in[0,R]^\calS$.}

\textbf{Parameter}: $\eta>0$.

\lForAll{$a\in\calS$}{
$
    p_{t+1}(a) = \frac{p_{t}(a)\e^{-\eta \ell_{t}(a)}}{\sum_{a'\in\calS} p_{t}(a')\e^{-\eta \ell_{t}(a')}}.
$
}
\Return{$\vc{p}_{t+1}$.}
\end{algorithm}
\end{figure}

The performance guarantee of the \hedge algorithm is captured by, e.g., Theorem~2.4 of \cite{cesa2006prediction}. The following lemma slightly extends their result for general values of $R$, which will be needed in the analysis in Section~\ref{sec:bandit}.

\begin{lemma}
    \label{lem: hedge bound approx}
    Algorithm~\ref{alg:hedge} ensures: 
$
        \E\left[\sum_{t=1}^T \ell_{t}(a_t)\right] \leq \frac{\eta R}{1-\e^{-\eta R}} \cdot \E\left[\min_{a^*\in\calS}\sum_{t=1}^T \ell_{t}(a^*)\right] + \frac{R\ln |\calS|}{1-\e^{-\eta R}}. 
$
\end{lemma}

\begin{proof}
    Let $w_{1}(a)=1$ for all $a\in\calS$ and define $w_{t+1}(a) = w_{t}(a)\e^{-\eta \ell_t(a)}$. Also, define $W_t=\sum_{a\in\calS} w_t(a)$. Then clearly, $p_t(a)\propto w_t(a)$ and $p_t(a)=w_t(a)/W_t$. 
    Using these definitions, we have 

    \begin{align*}
        \ln\frac{W_{T+1}}{W_1}=\sum_{t=1}^T \ln\frac{W_{t+1}}{W_t} = \sum_{t=1}^T \ln\frac{\sum_{a\in\calS} w_t(a)\e^{-\eta \ell_t(a)}}{W_t}=\sum_{t=1}^T \ln \left(\sum_{a\in\calS} p_t(a)\e^{-\eta \ell_t(a)}\right). 
    \end{align*}
    Since $\ell_t(a)\in[0,R]$ and $\e^{-\eta R x}$ is convex in $x$, we have
    \begin{align*}
        \e^{-\eta \ell_t(a)}\leq \frac{\ell_t(a)}{R}\e^{-\eta R} + \left(1-\frac{\ell_t(a)}{R}\right) = 1-\frac{1-\e^{-\eta R}}{R}\ell_t(a). 
    \end{align*}
    Thus, 
    \begin{align*}
        \ln \left(\sum_{a\in\calS} p_t(a)\e^{-\eta \ell_t(a)}\right) \leq \ln \left(1 - \frac{1-\e^{-\eta R}}{R}\sum_{a\in\calS} p_t(a)\ell_t(a)  \right)\leq - \frac{1-\e^{-\eta R}}{R}\sum_{a\in\calS} p_t(a)\ell_t(a),
    \end{align*}
    because $\ln (1-x)\leq -x$ for $x\geq 0$. 
    On the other hand, for any $a^* \in \calS$,
    \begin{align*}
        \ln \frac{W_{T+1}}{W_1} \geq  \ln \frac{w_{T+1}(a^*)}{W_1}  \geq \ln \frac{\e^{-\eta \sum_{t=1}^T  \ell_t(a^*)}}{|\calS|} = -\eta \sum_{t=1}^T \ell_t(a^*) - \ln |\calS|. 
    \end{align*}
    Combining both inequalities, we get 
    \begin{align*}
        \sum_{t=1}^T \sum_{a\in\calS} p_t(a)\ell_t(a) 
        \leq \frac{\eta R}{1-\e^{-\eta R}}\min_{a^* \in \calS}\sum_{t=1}^T \ell_t(a^*) + \frac{R\ln |\calS|}{1-\e^{-\eta R}}.  
    \end{align*}
    Taking expectation on both sides finishes the proof.
\end{proof}

%% file: full-information.tex

In this section, we consider two special cases in the full-information
setting; we obtain approximate regret bounds whose approximation ratio
depends only on $K$, the maximum cardinality of $\calA_t$.
These two special cases are the following:
1) in each round $t$, exactly one action has loss 0, i.e., for all $t$,
  $\sum_{a\in \calA_t} \one[\ell_t(a)=0] = 1$,
and 2) in each round $t$, exactly two actions have loss 0, i.e., for all $t$,
  $\sum_{a\in \calA_t} \one[\ell_t(a)=0] = 2$.
We remark again that these structures correspond to problems with sparse rewards, studied in previous work as well~\citep{kwon2016gains, bubeck2018sparsity}.

The first case is reminiscent of multi-class classification with 0-1 loss: there is only one ``label'' that is correct and incurs zero loss; other labels all incur a loss of one. In a typical classification problem, the learner uses \emph{features} as side information to infer labels; in our problem, we may view the available action set $\calA_t$ as the side information.  
For this case, in Section~\ref{subsec:hedgetree}, we give an algorithm called \hedgetree (Hedges Aggregated with Tournament Trees) which guarantees that the total loss of the learner is upper-bounded by $\order(\log_2 K)L^* + \order(N^2)$.

For the second case, in Section~\ref{subsec:hedgevote}, we design another (more involved) algorithm called \hedgevote (Hedges Over Pairs of Pairs) whose loss is upper-bounded by $\order(K^2) L^* + \order(N^4)$. Note that we get a worse approximation ratio in this case compared to the first case.

When the number of possible zero-loss actions 
exceeds $2$, it is not clear how to efficiently obtain an approximate regret bound where $\alpha$ is a function of $K$ and $\beta(T)$ is polynomial in $K$.  
However, an approximation ratio of $N$ is still achievable, even in the bandit setting, as shown in Section~\ref{sec:bandit}. 

The algorithms in Sections~\ref{subsec:hedgetree} and \ref{subsec:hedgevote} are based on similar ideas. They maintain several sub-algorithms, each dealing with a constant-size sub-problem (e.g., a 2-expert algorithm that compares the performance of actions $i,j$ in the rounds when they are both available). Then, when given $\calA_t$, a meta-algorithm aggregates the recommendations of these sub-algorithms and generates the final $a_t\in\calA_t$. The design of the sub-problems and their losses has the following two key properties: 

\begin{property}\label{prop:mistake}
   Whenever the learner makes a mistake (i.e., $\ell_t(a_t)=1$), there is at least one sub-algorithm which also makes a mistake in its sub-problem.
\end{property}   

\begin{property}\label{prop:no_mistake}
   Whenever the best ranking $\sigma$ makes no mistake (i.e., $\ell_t(\sigma(\calA_t))=0$), it also makes no mistake for all of the defined sub-problems. 
\end{property}   
   
These two properties are sufficient to ensure that algorithms with sub-linear regret for the sub-problems also guarantee good approximate regret bounds for the original problem.

\subsection{The \hedgetree Algorithm for One Zero-Loss Action}
\label{subsec:hedgetree}
\input{case-one}

\subsection{The \hedgevote Algorithm for Two Zero-Loss Actions}
\label{subsec:hedgevote}
\input{case-two}

%% file: case-one.tex

\begin{figure}[t]
\begin{algorithm}[H]
     \caption{\hedgetree (Hedges Aggregated with Tournament Trees)}
     \label{alg:hedge-tournament}
     \DontPrintSemicolon
     \lForAll{$i < j$}{
        set
     $p_{1}^{i,j}(i) =  p_{1}^{i,j}(j) = \frac{1}{2}$.
   }
   
     \For{$t=1, \ldots, T$}{
       Receive $\calA_t$ and let $(a_t, U_t) = \tournament (\calA_t, (p_{t}^{i,j})_{i,j})$.
       
       Choose $a_t$ and suffer loss $\ell_{t}(a_t)$.
         
       Learn $\ell_{t}(a)$ for all $a\in\calA_t$ and let $z_t$ be such that $\ell_t(z_t) = 0$.
       
       \lForAll{$i$ with $\{i, z_t\} \in U_t$}{
         $c_t^{i,z_t}(i) = 1$, \ $c_t^{i,z_t}(z_t)=0$, \ 
         $p_{t+1}^{i,z_t} = \updatehedge\left(p_t^{i,z_t}, c_t^{i,z_t}\right)$. 
       }
              
       \lForAll{other $i < j$}{
         let $c_t^{i,j}(\cdot)=0$ and $p_{t+1}^{i,j}=p_t^{i,j}$.
       }
     } 
\end{algorithm}

\begin{algorithm}[H]
    \caption{\tournament}
    \label{alg:selection-rule-1}
    \DontPrintSemicolon
    \KwIn{
      $\calA_t$: available action set at time $t$
      
      $P_t = (p_{t}^{i,j})_{i,j}$: distributions of hedges over all pairs $\{i,j\}$
    }
    \textbf{Initialization:} $U_t = \emptyset$.
    
    \lForAll{$i < j$}{sample $a_t^{i,j}\sim p_t^{i,j}$.}

    Let $\calT$ be a balanced binary tree with exactly $|\calA_t|$ leaves, each mapped to a distinct action in $\calA_t$.
    
    
    \lForEach{leaf $v$}{
      let $\winner(v)$ be the action $v$ is mapped to.
    }
    
    \ForEach{internal node $v$, in bottom-up order}{
         \lIf{$v$ has one child $v'$}{set $\winner(v) = \winner(v')$.}
         \lElse{let $i, j$ be the winners at the two children of $v$; set $\winner(v) = a_t^{i,j}$, and add $\{i, j\}$ to $U_t$.
           }
     }
    
    \Return{$\winner(\rt \text{ of } \calT)$, $U_t$.}
\end{algorithm}
\end{figure}

We begin with an algorithm for the case of a single zero-loss action per round.
Recall that the sleeping experts algorithm by \cite{kleinberg2010regret} is based on the idea of ``hedging over all rankings'' --- that is, viewing each ranking of actions as an ``expert'' in \hedge. This leads to (exact) regret bounds with respect to the best ranking, but requires keeping track of $N!$ experts in total. 
Instead of keeping track of an expert for each permutation, our algorithm only maintains one expert for each pair of actions. This results in a coarser representation, but we show that it still achieves good guarantees. In other words, while \cite{kleinberg2010regret} maintains one algorithm that learns over exponentially many experts, we maintain ${N \choose 2}$ \hedge algorithms, each learning over \textit{two} actions. Then, a meta algorithm combines the recommendations of all 2-expert \hedge algorithms and decides on the final action the learner should choose. 

To learn the preference between the pair of actions $\{i,j\}\subset [N]$ with $i\neq j$, \hedgetree simply runs an instance $\calH_{i,j}$ of \hedge (Algorithm~\ref{alg:hedge}) with $\calS=\{i,j\}$.
\hedgetree then uses the following \textit{tournament} approach as the meta algorithm to combine the recommendations of all \hedge algorithms. In each round $t$, \hedgetree creates a single-elimination tournament tree 
$\calT_t$ with $|\calA_t|$ leaves, and thus depth $1 + \lceil \log_2 (|\calA_t|) \rceil$. It assigns each element in $\calA_t$ to one leaf of $\calT_t$ (arbitrarily). Then the actions perform a single-elimination tournament following $\calT_t$ to generate the final winner $a_t$. For each pair of actions $(i,j)$, the winner and loser are determined by the \hedge algorithm $\calH_{i,j}$. Notice that each action is involved in at most $\log_2 K$ comparisons in each round. We will show that this is the regret approximation ratio of \hedgetree.

More formally, in Algorithm~\ref{alg:hedge-tournament},  $p_t^{i,j}$ denotes the $p_t$ maintained by the \hedge instance $\calH_{i,j}$; we use $p_{t}^{i,j}(i)$ and $p_{t}^{i,j}(j)$ to denote the probabilities for the actions $i$ and $j$, respectively.
Note that $p_t^{i,j}$ is shorthand for $p_t^{\{i,j\}}$, so $p_t^{i,j}$ and $p_t^{j, i}$ are always the same, and we only run one instance of \hedge for each pair $\{i,j\}$ (similarly for the notation $c_t^{i,j}$ and $a_t^{i,j}$ below).
In Algorithm~\ref{alg:selection-rule-1}, each \hedge instance $\calH_{i,j}$ samples a winner $a_t^{i,j}$ according to $p_t^{i,j}$, and a tournament is run.
In this process, a set $U_t$ is used to record all pairs involved in the tournament.

After choosing the final winner $a_t$ of the tournament, \hedgetree receives the loss feedback. We let $z_t$ denote the unique zero-loss action; hence, for all $a\in \calA_t \setminus \{z_t\}$, the loss is $\ell_t(a)=1$.
Then, for all pairs in $U_t$ that involve $z_t$, the algorithm updates the corresponding \hedge instance with the natural loss vector: action $z_t$ has loss $0$, and the other action has loss $1$.
For all other pairs $\{i,j\}$, the algorithm does not make any updates, although for notational convenience in the analysis, we still define a loss vector $c_t^{i,j}$ to be the all-zero vector, so that $p_{t+1}^{i,j} = p_{t}^{i,j} = \updatehedge\big(p_t^{i,j}, c_t^{i,j}\big)$ holds.


The performance of \hedgetree is summarized in the following theorem: 

\begin{theorem} \label{theorem: HATT}
    \hedgetree (Algorithm~\ref{alg:hedge-tournament}) guarantees that
    \begin{align*}
\E\left[\sum_{t=1}^T \ell_t(a_t)\right]
      \leq  \frac{\eta (1+\lceil \log_2(K) \rceil) }{1-\e^{-\eta}} \E\left[\min_\sigma \sum_{t=1}^T \ell_t(\sigma(\calA_t))\right] + {N\choose 2}\cdot \frac{\ln 2}{(1-\e^{-\eta})}.  
    \end{align*}
    In particular, when $\eta=1$, the above is no more than
$
        \order(\log_2(K)) \cdot \E\left[\sum_{t=1}^T \ell_t(\sigma(\calA_t))\right] + \order\left( N^2 \right). 
$
    
\end{theorem}

Note that the approximation ratio is only logarithmic in $K$, and the additive regret term is also independent of $T$. The proof of Theorem~\ref{theorem: HATT} can be obtained by directly combining the following three lemmas ($a_t^{i,j}$ and $c_t^{i,j}$ are as defined in Algorithm~\ref{alg:selection-rule-1} and Algorithm~\ref{alg:hedge-tournament}, respectively).
Lemmas~\ref{lemma: algorithm 1 ineq 1} and~\ref{lemma: alorithm 1 ineq 3} assert that \hedgetree ensures Properties~\ref{prop:mistake} and~\ref{prop:no_mistake}, respectively.

\begin{lemma}\label{lemma: algorithm 1 ineq 1}
    In Algorithm~\ref{alg:hedge-tournament}, whenever the learner makes a mistake (i.e., $\ell_t(a_t)=1$), there must be a \hedge algorithm which also makes a mistake. More formally, for every $t$, 
    \begin{align*}
        \ell_t(a_t) \leq \sum_{i < j} c_t^{i,j}(a_t^{i,j}).  
    \end{align*}
\end{lemma}

\begin{proof}
    If $\ell_t(a_t)=0$, then the inequality clearly holds. If $\ell_t(a_t)=1$, by the tournament approach, there must exist an $i\neq z_t$ with $\{i, z_t\} \in U_t$ and $a_t^{i,z_t}=i$. Thus we have
    \begin{align*}
        c_t^{i,z_t}(a_t^{i,z_t}) 
        = c_t^{i,z_t}(i) 
        = \ell_t(i)\one[\{i, z_t\} \in U_t] = 1. 
    \end{align*}
    Thus the inequality also holds when $\ell_t(a_t)=1$. 
\end{proof}
  
\begin{lemma}\label{lemma: alorithm 1 ineq 2}
Algorithm~\ref{alg:hedge-tournament} guarantees that for all $i < j$,
\begin{align*}
    \scalebox{0.9}{$\displaystyle\E\left[\sum_{t=1}^T c_{t}^{i,j}(a_t^{i,j})\right]  \leq \frac{\eta}{1-\e^{-\eta}} \E\left[\min_\sigma \sum_{t=1}^T c_{t}^{i,j}(\sigma(i,j))\right] + \frac{\ln 2}{1-\e^{-\eta}}.  $}
\end{align*}
\end{lemma}

\begin{proof}
Note that importantly, the value of $c_t^{i,j}$ is decided independently of $a_t^{i,j}$ (although it could depend on other $a_t^{i',j'}$).
We can therefore apply Lemma~\ref{lem: hedge bound approx} with $R=1$ and $\calS=\{i,j\}$, which proves the lemma.
\end{proof}

\begin{lemma}\label{lemma: alorithm 1 ineq 3}
    Algorithm~\ref{alg:hedge-tournament} guarantees that for all rankings $\sigma$,  
    \begin{align*}
        \sum_{i < j} c_t^{i,j}(\sigma(i,j)) \leq (1+\lceil \log_2(K) \rceil) \cdot \ell_t(\sigma(\calA_t)). 
    \end{align*}
\end{lemma}

\begin{proof}
    If $\ell_t(\sigma(\calA_t))=0$, then $\sigma(\calA_t)=z_t$ (i.e., $\sigma$ ranks $z_t$ first among $\calA_t$), and thus $\sigma(i,z_t)=z_t$ for all $i\in\calA_t$. Therefore, 
    \begin{align*}
        \sum_{i < j} c_t^{i,j}(\sigma(i,j)) 
        &= \sum_{i\in \calA_t, i\neq z_t} c_t^{i,z_t}(\sigma(i,z_t)) = \sum_{i\in \calA_t, i\neq z_t} c_t^{i,z_t}(z_t)=0. 
    \end{align*}
    If $\ell_t(\sigma(\calA_t))=1$, then
    \begin{align*}
        \sum_{i< j} c_t^{i,j}(\sigma(i,j)) 
        \leq \sum_{i\in\calA_t i\neq z_t} \one[\{i, z_t\} \in U_t] \leq 1+\lceil \log_2(K) \rceil. 
    \end{align*}
    In both cases, 
    \begin{align*}
        \sum_{i< j} c_t^{i,j}(\sigma(i,j))\leq (1+\lceil \log_2(K) \rceil) \cdot \ell_t(\sigma(\calA_t)). 
    \end{align*}
  \end{proof}

We are now ready to prove the theorem.

\begin{proof}[of Theorem~\ref{theorem: HATT}]
     We apply Lemmas~\ref{lemma: algorithm 1 ineq 1}, \ref{lemma: alorithm 1 ineq 2}, \ref{lemma: alorithm 1 ineq 3} successively: 
     \begin{align*}
\E\left[ \sum_{t=1}^T \ell_t(a_t) \right] 
          &\leq \E\left[ \sum_{i < j}\sum_{t=1}^T   c_t^{i,j}(a_t^{i,j}) \right] \\
          &\leq \frac{\eta}{1-\e^{-\eta}}\E\left[  \min_\sigma \sum_{i < j} \sum_{t=1}^T c_t^{i,j}(\sigma(i,j)) \right] + \sum_{i < j} \frac{\ln 2}{1-\e^{-\eta}}  \\
          &\leq  \frac{\eta (1+\lceil \log_2(K) \rceil)}{1-\e^{-\eta}}  \E\left[\min_\sigma\sum_{t=1}^T  \ell_t(\sigma(\calA_t))\right]
           + {N \choose 2} \cdot \frac{\ln 2}{1-\e^{-\eta}}. 
     \end{align*}
     This completes the proof.
\end{proof}
   
Finally, we point out that even in this simple case with one zero-loss action, achieving no-regret performance (i.e. $\alpha=1$) is still as hard as PAC-learning DNFs, as shown below.

\begin{theorem}\label{thm:hardness}
If there exists a computationally efficient no-regret algorithm for the sub-class of sleeping expert problems which always have exactly one zero-loss action, then there exists a computationally efficient algorithm for PAC-learning DNFs under arbitrary distributions. 
\end{theorem}

We do not have a better lower bound on the approximation ratio for polynomial-time algorithms; these kinds of computation-constrained lower bounds are scarce in the literature.
However, we note that, together with~\citep{awasthi2010improved}, our proof of Theorem~\ref{thm:hardness} implies that achieving an approximation ratio better than $\order(K^{1/3})$ in the general case would improve the state-of-the-art for agnostically learning disjunctions with polynomial-time algorithms.

\begin{proof}
Our hardness proof is heavily based on the hardness result in \cite{kanade2014learning}. They reduce from PAC-learning of DNFs to \emph{agnostic learning of disjunctions}, and from that problem to achieving no-regret performance with high probability against the best ranking in sleeping expert problems. 

The key observation is that the instances of the sleeping expert problem produced by the reduction in \cite{kanade2014learning} are already almost of the restricted form of Theorem~\ref{thm:hardness}: (1) the set of available actions always satisfies $|\calA_t|=K$, (2) the losses $\ell_t(a)\in \{0,1\}$ are always binary, and (3) the loss vector $\vc{\ell}_t$ always has exactly one 0 or exactly one 1. Only the third property is different from our model of exactly one 0. Our proof therefore provides a reduction from their instances to ours.

Let $\calE^0_t = [\sum_{a\in\calA_t} \one[\ell_t(a)=0] = 1]$ be the event that the loss vector in round $t$ has exactly one zero,
and $\calE^1_t = [\sum_{a\in\calA_t} \one[\ell_t(a)=1] = 1]$ the event that the loss vector in round $t$ has exactly one one.

Assume that there is an algorithm $\calZ_0$ which always achieves no regret for instances in which all loss vectors have exactly one zero.
That is, for any binary-loss sequence $\vc{\ell}_t$ that satisfies $\calE^0_t$ for all $t$, the algorithm $\calZ_0$ outputs $a_1, \ldots, a_T$ such that for all $\sigma$, 
\begin{align*}
    \E\left[ \sum_{t=1}^T \ell_t(a_t) - \sum_{t=1}^T \ell_t(\sigma(\calA_t)) \right] = o(T). 
\end{align*} 
We will give a reduction showing how to leverage $\calZ_0$ to obtain an algorithm $\calZ_{01}$ which achieves the same no-regret guarantee for instances in which all loss vectors have exactly one zero or exactly one one.
The algorithm $\calZ_{01}$ works as follows.

\begin{itemize}
     \item Upon receiving the available action set $\calA_t$, $\calZ_{01}$ passes $\calA_t$ to $\calZ_0$, and chooses the action $a_t\in\calA_t$ returned by $\calZ_0$. 
     \item The algorithm observes losses $\ell_t'(a)$ for all $a\in\calA_t$, and can determine which of $\calE^0_t, \calE^1_t$ holds.
     \begin{itemize}
         \item If $\calE^0_t$ holds, then with probability $\frac{1}{K-1}$, $\calZ_{01}$ sets $\vc{\ell}_t$ to be $\vc{\ell}_t'$; with the remaining probability $\frac{K-2}{K-1}$, it uniformly randomly draws $z_t$ from $\calA_t$, sets $\ell_t(z_t)=0$, and $\ell_t(a)=1$ for all $a\in \calA_t\setminus \{z_t\}$.
         \item If $\calE^1_t$ holds, then $\calZ_{01}$ uniformly randomly draws $z_t$ from the $(K-1)$ zero-loss actions. It sets $\ell_t(z_t)=0$ and $\ell_t(a)=1$ for all $a\in \calA_t\setminus \{z_t\}$.   
     \end{itemize}
     \item $\calZ_{01}$ then passes the loss vector $\vc{\ell}_t$ to $\calZ_{0}$.
\end{itemize}

The loss vectors $\vc{\ell}_t$ always have exactly one zero entry. The expected losses are as follows:

\begin{itemize}
    \item Conditioned on $\calE^0_t$, we have $\E[\ell_t(a)] = \frac{1}{K-1}\cdot \ell_t'(a) + \frac{K-2}{K-1}\cdot \frac{K-1}{K} = \frac{K-2}{K} + \frac{\ell_t'(a)}{K-1}$. 
    \item Conditioned on $\calE^1_t$, we have $\E[\ell_t(a)] = \frac{K-2}{K-1}\cdot \one[\ell_t'(a)=0] + 1\cdot \one[\ell_t'(a)=1] = \frac{K-2}{K-1} + \frac{\ell_t'(a)}{K-1}$.
\end{itemize}

Therefore, 
\begin{align*}
&\frac{1}{K-1} \cdot \E\left[ \sum_{t=1}^T \ell_t'(a_t)\right] - \frac{1}{K-1} \cdot \sum_{t=1}^T\ell_t'(\sigma(\calA_t)) \\
&=\E\left[ \sum_{t=1}^T \ell_t(a_t)-\frac{(K-2) \cdot \one[\calE^0_t]}{K}-\frac{(K-2) \cdot \one[\calE^1_t]}{K-1}  \right] \\
&\quad\quad - \E\left[ \sum_{t=1}^T \ell_t(\sigma(\calA_t))-\frac{(K-2) \cdot \one[\calE^0_t]}{K}-\frac{(K-2)\cdot \one[\calE^1_t]}{K-1}  \right] \\
&=\E\left[\sum_{t=1}^T \ell_t(a_t)\right] - \E\left[\sum_{t=1}^T \ell_t(\sigma(\calA_t))\right] = o(T), 
\end{align*}
where the last line is guaranteed by our assumption that $\calZ_0$ is no-regret. Multiplying by $K-1$, we also obtain that
\begin{align*}
    \E\left[ \sum_{t=1}^T \ell_t'(a_t)-\ell_t'(\sigma(\calA_t))\right] = o(T). 
\end{align*}
To finish the reduction from the case of \cite{kanade2014learning} to our case, we need to further argue that the algorithm $\calZ_{01}$ with sublinear expected regret can be transformed into an algorithm that has sublinear regret with high probability.

To see this, one simply runs $T$ copies of $\calZ_{01}$ simultaneously and aggregates them via Hedge to decide the final output. By Hoeffding's inequality, with probability at least $1-\delta$, one of the $T$ copies must have regret smaller than its expectation plus $\order(\sqrt{T\ln(1/\delta)})$ (since the range of regret is $[-T,T]$). Also note that Hedge itself has $\order(\sqrt{T\ln(T/\delta)})$ regret against any one of the copies with probability $1-\delta$. Combining these two statements, we have thus constructed a new algorithm which has sublinear regret with high probability. This completes the proof. 
\end{proof}

%% file: case-two.tex

The case of two zero-loss actions is significantly more complicated. Again, we want to design sub-problems with Properties~\ref{prop:mistake} and~\ref{prop:no_mistake}.  
To achieve these properties, it is now not sufficient any more to define sub-problems comparing only two actions, as we did in Section~\ref{subsec:hedgetree}. This is because it is now possible that a ranking makes no mistake ($\ell_t(\sigma(\calA_t))=0$), while making mistakes in some pairwise comparisons ($\ell_t(\sigma(i,j))=1$). For example, consider the case when the first, second, and third actions according to the ranking $\sigma$ have losses $0,1,0$, respectively. Then $\sigma$ does not make a mistake in this round because its top choice receives zero loss. However, in the sub-problem that compares the second and third actions, $\sigma$ does make a mistake because its choice among the two actions incurs a loss of $1$. This would violate Property~\ref{prop:no_mistake}. 

To address the above issue, we design sub-problems as ``comparing two pairs of actions,'' as well as ``choosing among three actions.''
The hedges for triples of actions are standard. For each set $S \subseteq \calA$ with $|S| = 3$, there is a separate \hedge that recommends one of the three actions in $S$. 
This instance is updated only when $S \subseteq \calA_t$ turns out to contain both zero-loss actions, in which case the loss vector is the natural one following $\ell_t$.
See the last part of Algorithm~\ref{alg:for-two-zero}.

The subproblems for pairs of actions are more intricate and non-standard, and we next explain them in detail.
Each such sub-problem compares a pair $X=\{i,j\}$ of actions with another pair $Y=\{k,l\}$, where $i,j,k,l$ are all distinct. The algorithm \hedgevote uses a separate 2-expert \hedge $\calH_{X,Y}$ to learn each such sub-problem $(X,Y)$. 
This instance is only updated when both $X$ and $Y$ are in $\calA_t$.
In this case, only when one of $X$ or $Y$ consists of both of the two zero-loss actions do we assign positive loss to the other pair. 
More precisely, if $\ell_t(i)=\ell_t(j)=0$, then choosing $Y$ in this sub-problem incurs a loss of $1$; similarly, if $\ell_t(k)=\ell_t(l)=0$, then choosing $X$ incurs a loss of $1$. 
In all other cases, we define both actions' losses as $0$. 
We also define the \emph{choice} of a ranking $\sigma$ for this sub-problem as follows: if $\sigma(i,j,k,l)\in X$, then the choice of $\sigma$ is $X$; otherwise, it is $Y$. 
This way, when a ranking $\sigma$ makes no mistake in the original problem ($\ell_t(\sigma(\calA_t))=0$), it also has zero loss in all sub-problems. 
This ensures that Property~\ref{prop:no_mistake} holds.
(The preceding arguments are formalized in Lemma~\ref{algorithm 2 ineq 3}.)  

\begin{figure}[t]
  \begin{algorithm}[H]
    \caption{\hedgevote (Hedges Over Pairs of Pairs)}
    \label{alg:for-two-zero}
    \DontPrintSemicolon

    \lForAll{pairs $X, Y$ with $X \cap Y = \emptyset$}{
     set $p_1^{X,Y}(X)=p_1^{X,Y}(Y)=\frac{1}{2}$.
   }
   
   \lForAll{triples $S=\{i,j,k\}$ of actions}{
     set $q_1^S(i) = q_1^S(j) = q_1^S(k) = \frac{1}{3}$. 
   }
   
    \For{$t=1, \ldots, T$}{
      Receive $\calA_t$ and let 
            $a_t = \vote\left(\calA_t, (p_t^{X,Y})_{X, Y}, (q_t^S)_S  \right)$. 

      Choose $a_t$, suffer loss $\ell_t(a_t)$, and learn $\ell_t(a)$ for all $a\in\calA_t$.

      Let $Z_t=\{a: \ell_t(a)=0\}$ be the pair of actions with zero loss.

      \ForAll{disjoint pairs $X, Y$}{
      Define 
      $c_t^{X,Y}(X) = \one[Z_t=Y \text{ and } X \subseteq \calA_t]$ and $c_t^{X,Y}(Y) = \one[Z_t=X \text{ and } Y \subseteq \calA_t]$. 
                
      Update
       $p_{t+1}^{X,Y} = \updatehedge\left( p_t^{X,Y}, c_t^{X,Y} \right)$.
     }
     
     \ForAll{triples $S$}{
            Define 
                $d_t^S(i) = \ell_t(i) \cdot \one[Z_t\subseteq S \subseteq \calA_t], \;\forall i\in S$.

            Update
                $q_{t+1}^S = \updatehedge\left( q_t^S, d_t^S \right)$. 
           
     }
  }
\end{algorithm}
\begin{algorithm}[H] 
    \caption{\vote}
    \label{alg:selection-rule-2}
    \DontPrintSemicolon
    \KwIn{ $\calA_t$: available action set at time $t$

      $(p_t^{X,Y})_{X, Y}$: hedge probabilities for all disjoint pairs
   
    $(q_t^S)_S$: hedge probabilities for all triples $S$
    }
    
    \textbf{Initialization:}
    
    \lForAll{distinct pairs $X, Y$}{sample $A_t^{X,Y} \sim p_t^{X,Y}$.}
    
    \lForAll{triples $S$}{sample $b_t^S \sim q_t^S$.}
    
    Pair $X \subseteq \calA_t$ is a \emph{good pair} if $A_t^{X,Y}=X$ for all $Y \subseteq \calA_t$ such that $X \cap Y = \emptyset$. 
    
    \lIf{there is no good pair}{arbitrarily choose an $a_t\in\calA_t$.}
    \lElseIf{there is a common action in all good pairs}{
      let $a_t$ be such a common action.
    }
    \lElse{there are exactly three good pairs of the form $\{i,j\}, \{j,k\}, \{k,i\}$; let $a_t=b_t^{\{i,j,k\}}$.}
    \Return{$a_t$.}
\end{algorithm}
\end{figure}

To make Property~\ref{prop:mistake} also hold, we design complex rules for aggregating the recommendations of all hedges so that every time the learner suffers loss $1$ in the original problem, it must also suffer positive loss in some sub-problem. For this purpose, we define \emph{good pairs} in the sub-algorithm \vote (Algorithm~\ref{alg:selection-rule-2}). A good pair $X \subseteq \calA_t$ is a pair such that for all disjoint pairs $Y \subseteq \calA_t$, the hedge $\calH_{X,Y}$ chooses $X$ as the winner. It is possible that no pair is good, or that more than one pair is good. For each possibility, we discuss how to choose the final $a_t$ (see Algorithm~\ref{alg:selection-rule-2}). 
The following lemma shows that Algorithm~\ref{alg:selection-rule-2} indeed considers all cases.

\begin{lemma} \label{lemma: good pair}
      For the good pairs defined above, the following hold: 
         1) Any two good pairs must have one common action;
         2) Either all good pairs have one common action, or there are exactly three good pairs, and they are of the form $\{i,j\}, \{j,k\}, \{k,i\}$. 
\end{lemma}
\begin{proof}
    If $X$, $Y$ were disjoint, then for $X$ to be good, $\calH_{X,Y}$ has to choose $X$, but for $Y$ to be good, $\calH_{X,Y}$ has to choose $Y$. So $X$, $Y$ must intersect. 
    This also directly implies the second statement.
\end{proof}

The case when there are exactly three good pairs of the form $\{i,j\}, \{j,k\}, \{k,i\}$ is the only case in which the algorithm needs to also consult the hedges over triples.
The approximate regret guarantee of \hedgevote is given by the following theorem. 

\begin{theorem} \label{theorem: HOPP algorithm}
    \hedgevote ensures: 
$
        \E\left[\sum_{t=1}^T \ell_t(a_t)\right] \leq \order(K^2) \E\left[\min_\sigma \sum_{t=1}^T \ell_t(\sigma(\calA_t))\right] + \order(N^4). 
$
\end{theorem}

Note that the approximation ratio $\order(K^2)$ is significantly worse than the case with one zero-loss action, but is still only a function of $K$ (and not $N$).
The additive regret term is also worse, but still independent of $T$.
To prove Theorem~\ref{theorem: HOPP algorithm}, we make use the following three lemmas (the notation in the lemmas is defined in Algorithms~\ref{alg:for-two-zero} and \ref{alg:selection-rule-2}).
Again, Lemmas~\ref{algorithm 2 ineq 1} and~\ref{algorithm 2 ineq 3} assert that \hedgevote satisfies Properties~\ref{prop:mistake} and~\ref{prop:no_mistake}.

\begin{lemma}\label{algorithm 2 ineq 1}
    \hedgevote guarantees that
    \begin{align*}
        \ell_t(a_t) \leq \sum_{X, Y \text{ disjoint}} c_t^{X,Y}(A_t^{X,Y}) + \sum_{S: |S| = 3} d_t^S(b_t^S). 
    \end{align*}
\end{lemma}

\begin{proof}
    If $\ell_t(a_t)=0$, then the inequality clearly holds. Therefore, we only need to consider the case $\ell_t(a_t)=1$. 

First, for all the cases except when there are exactly three good pairs of the form $\{i,j\}, \{j,k\}, \{k,i\}$, we prove that the pair $Z_t$ of zero-loss actions cannot be good:
    \begin{itemize}
        \item If there is no good pair, then clearly $Z_t$ cannot be good.
        \item If there is exactly one good pair, then $Z_t$ would be that pair. Therefore, the algorithm would have selected an element of $Z_t$, implying that $\ell_t(a_t) = 0$, a contradiction.
        \item If all good pairs have one common action, and $Z_t$ is one of them, then the algorithm selects an element in the intersection of the good pairs. In particular, the element $a_t \in Z_t$, so $\ell_t(a_t) = 0$, a contradiction.
    \end{itemize}

    Since $Z_t$ is not a good pair, there exists a pair $X \subseteq \calA_t$ such that $A_t^{X,Z_t}=X$, and thus 
    \begin{align*}
        c_t^{X,Z_t}(A_t^{X,Z_t}) = c_t^{X,Z_t}(X) = 1,
    \end{align*}
proving the lemma statement.
The only remaining case is when there are exactly three good pairs $\{i,j\}, \{j,k\}, \{k,i\}$.
If $Z_t$ is not one of these pairs, then the exact same argument holds; otherwise, since $\ell_t(a_t)=1$, we must have $Z_t=\{i,j\}$ and $a_t=k$
and therefore, 
    \begin{align*}
        d_t^{i,j,k}(b_t^{i,j,k}) = d_t^{i,j,k}(k) = \ell_t(k)=1,
    \end{align*}
    finishing the proof.
\end{proof}
  
\begin{lemma}\label{algorithm 2 ineq 2}
    \hedgevote ensures that for all disjoint pairs $X, Y$,
    \begin{align*}
        \E\left[\sum_{t=1}^T c_t^{X,Y}(A_t^{X,Y})\right]
        &\leq \frac{\eta}{1-\e^{-\eta}}\E\left[\min_\sigma \sum_{t=1}^T c_t^{X,Y}(\sigma(X,Y))\right] + \frac{\ln 2}{1-\e^{-\eta}},    
    \end{align*}
    where $\sigma(X,Y)= X$ if $\sigma(X \cup Y) \in X$ and $\sigma(X,Y)= Y$ otherwise. 
    Also, for all triples $S$,
    \begin{align*}
        \E\left[\sum_{t=1}^T d_t^S(b_t^S)\right]
        & \leq \frac{\eta}{1-\e^{-\eta}} \E\left[\min_\sigma \sum_{t=1}^T d_t^S(\sigma(S))\right]+\frac{\ln 3}{1-\e^{-\eta}}. 
    \end{align*}
\end{lemma}

\begin{proof}
Note that the value of $c_t^{X,Y}$ is independent of $A_t^{X,Y}$, and the value of $d_t^S$ is independent of $b_t^S$.
Therefore, the first bound is obtained by applying Lemma~\ref{lem: hedge bound approx} with $R=1$ and $\calS=\{X,Y\}$, and the second bound by applying the same lemma with $R=1$ and $\calS=S$.
\end{proof}

\begin{lemma}\label{algorithm 2 ineq 3}
    \hedgevote ensures that for all rankings $\sigma$,  
    \begin{align*}
        \sum_{X, Y \text{ disjoint}} c_t^{X,Y}(\sigma(X,Y)) & \leq {K-2 \choose 2}  \cdot \ell_t(\sigma(\calA_t)) \;\text{and}\;
         \sum_{S: |S| = 3} d_t^S(\sigma(S)) & \leq (K-2) \cdot \ell_t(\sigma(\calA_t)). 
    \end{align*}
\end{lemma}

\begin{proof}
    If $\ell_t(\sigma(\calA_t))=0$, then $\sigma(X,Z_t)=Z_t$ for every $X \subseteq \calA_t$. Also, $\sigma(Z_t\cup\{i\})\in Z_t$ for every $i \in \calA_t$. Therefore, by the construction of $c_t^{X,Y}$ and $d_t^S$, we have
    \begin{align*}
        \sum_{X,Y \text{ disjoint}} c_t^{X,Y}(\sigma(X,Y))  &=\sum_{X\subseteq \calA_t \setminus Z_t} c_t^{X,Z_t}(\sigma(X,Z_t))=\sum_{X\subseteq \calA_t \setminus Z_t} c_t^{X,Z_t}(Z_t)=0.\\
        \sum_{S: |S| = 3} d_t^S(\sigma(S))
        &=\sum_{i \in \calA_t \setminus Z_t} d_t^{Z_t\cup\{i\}}(\sigma(Z_t\cup\{i\}))=0. 
    \end{align*}
    When $\ell_t(\sigma(\calA_t))=1$, we have $\sum_{X, Y \text{ disjoint}} c_t^{X,Y}(\sigma(X,Y))=\sum_{X\subseteq \calA_t \setminus Z_t} 1 \leq {K-2 \choose 2}= {K-2 \choose 2} \cdot \ell_t(\sigma(\calA_t))$, proving the first inequality. For the second inequality, we use $\sum_S d_t^S(\sigma(S)) \leq \sum_{i \in \calA_t \setminus Z_t} 1 \leq K-2 = (K-2) \cdot \ell_t(\sigma(\calA_t))$. 
\end{proof}

We are now ready to prove the theorem.

\begin{proof}[of Theorem~\ref{theorem: HOPP algorithm}]
    We apply Lemmas~\ref{algorithm 2 ineq 1}, \ref{algorithm 2 ineq 2}, \ref{algorithm 2 ineq 3} successively: 
    \begin{align*}
         \E\left[\sum_{t=1}^T \ell_t(a_t) \right]
         &\leq \E\scalebox{0.85}{$\displaystyle\left[\sum_{t=1}^T \left(\sum_{X,Y} c_t^{X,Y}(A_t^{X,Y}) + \sum_{S: |S|=3} d_t^S(b_t^S)\right)\right]$} \\
         &\leq \frac{\eta}{1-\e^{-\eta}}\E\scalebox{0.85}{$\displaystyle\left[\min_\sigma \sum_{t=1}^T \left(\sum_{X,Y} c_t^{X,Y}(\sigma(X,Y)) + \sum_{S: |S|=3} d_t^S(\sigma(S))\right)\right]$}
          + \order\left(\frac{N^4}{1-\e^{-\eta}}\right) \\
         &\leq \frac{\eta}{1-\e^{-\eta}} \cdot \E\scalebox{0.85}{$\displaystyle\left[\order(K^2)\min_\sigma\sum_{t=1}^T \ell_t(\sigma(\calA_t)) \right]$} + \order\left(\frac{N^4}{1-\e^{-\eta}}\right).
     \end{align*}   
     Picking $\eta=1$ completes the proof. 
\end{proof}
   
%
%
%
%

%% file: bandit.tex

For the bandit setting, we consider two regimes. The first is the setting of Section~\ref{subsec:hedgetree}, i.e., in each round, exactly one action has zero loss. We show how to adapt Algorithm~\ref{alg:hedge-tournament} to the bandit setting while maintaining the same $\order(\log K)$ approximation ratio, albeit at the cost of larger additive regret.
Then, we consider the bandit model without any assumptions on the sizes of available action sets or numbers of zero-loss actions. In this case, we give an algorithm with approximation ratio $\order(N)$.

\paragraph{Bandit-\hedgetree.}
We begin by considering the setting of Section~\ref{subsec:hedgetree}, i.e., in each round $t$, exactly one action $z_t$ has loss 0, while all others have loss 1. We show how to combine the ideas of Algorithm~\ref{alg:hedge-tournament} with the ``inverse-propensity weighting'' technique to turn the algorithm into a bandit algorithm.

Since the algorithm does not learn the loss of all actions, we cannot define $c_t^{i,j}$ as in Algorithm~\ref{alg:hedge-tournament}. However, notice that when the learner happens to draw the zero-loss action at time $t$ (i.e., $\ell_t(a_t)=0$), she can infer all other actions' losses. Based on this observation, we can define an unbiased estimator for the $c_t^{i,j}$ in Algorithm~\ref{alg:hedge-tournament}. First, we define an exploration indicator $\rho_t$, which is drawn independently in each round $t$, and is 1 with probability $\mu$ and 0 otherwise. If $\rho_t=1$, then $a_t$ is drawn uniformly randomly from $\calA_t$; otherwise, $a_t$ is set to the output of Algorithm~\ref{alg:selection-rule-1} (as in the full-information setting). Then, we define 
$
    c_t^{i,j}(i) = \ell_t(i) \cdot \frac{|\calA_t| \cdot \one[\rho_t=1] \cdot \one[\ell_t(a_t)=0]}{\mu}  
$
if $a_t \in\{i,j\}\in U_t$; otherwise, $c_t^{i,j}(i)=0$.
This number is always accessible because when $\ell_t(a_t)=0$, the learner can infer the losses of all actions. Note that the $|\calA_t|\cdot \one[\rho_t=1] \cdot \one[\ell_t(a_t)=0]/\mu$ factor has an expectation of $1$ because $\rho_t=1$ happens with probability $\mu$, and when $\rho_t=1$, $a_t=z_t$ with probability $1/|\calA_t|$. So we see that the $c_t^{i,j}$ in Algorithm~\ref{alg:exp3-tournament} are exactly unbiased estimators for the $c_t^{i,j}$ defined in Algorithm~\ref{alg:hedge-tournament}. 

Note that the scaling by $\mu$ in the definition of $c_t^{i,j}$ results in values that are not in $[0,1]$; this is why we needed the more general bound of Lemma~\ref{lem: hedge bound approx} for the analysis of Hedge.

\begin{algorithm}[t!]
     \caption{Bandit-\hedgetree}
     \label{alg:exp3-tournament}
     \DontPrintSemicolon
     \lForAll{$i<j$}{set $p_{1}^{i,j}(i) =  p_{1}^{i,j}(j) = \frac{1}{2}$.}

     \For{$t=1, \ldots, T$}{
       Receive $\calA_t$.

       Let $(\widehat{a}_t, U_t) = \tournament (\calA_t,
       (p_{t}^{i,j})_{i,j})$.

       Draw $\rho_t \sim \text{Bernoulli}(\mu)$.
         
         \lIf{$\rho_t=1$}{let $a_t \sim \text{Uniform}(\calA_t)$ \textbf{else} let $a_t=\widehat{a}_t$.}

         Choose $a_t$ and suffer loss $\ell_{t}(a_t)$.

         \If{$\rho_t=1$ and $\ell_t(a_t)=0$ }{ 
           \texttt{/\!\!/{\small\ In this case, $z_t=a_t$.}}
           
            \ForAll{$i$ with $\{i, z_t\} \in U_t$}{   
              $c_t^{i,z_t}(i) = \frac{|\calA_t|\cdot \one[\rho_t=1] \cdot \one[\ell_t(a_t)=0]}{\mu},~ c_t^{i,z_t}(z_t)=0$,~ $p_{t+1}^{i,z_t} = \updatehedge\left(p_t^{i,z_t}, c_t^{i,z_t} \right)$. 
            }
              
            \lForAll{other $i<j$}{let $c_t^{i,j}(\cdot)=0$ and $p_{t+1}^{i,j}=p_t^{i,j}$.
            }
         }
         \lElse{\lForAll{$i < j$}{let $c_t^{i,j}(\cdot)=0$ and $p_{t+1}^{i,j}=p_t^{i,j}$.}}
         }
\end{algorithm}

Also note that the way we construct the estimators is different from the standard way for the multi-armed bandit problem~\citep{auer2002nonstochastic}, i.e., the special case when $\calA_t$ is fixed for all $t$.
The standard way would require computing the exact probability of choosing each action, which is complicated for our algorithm.
Moreover, for our problem, to design algorithms with Properties~\ref{prop:mistake} and~\ref{prop:no_mistake}, it is also important to assign non-zero losses to Hedges only when we know exactly what the loss vector is.
This is also the reason that we are unable to generalize \hedgevote to the bandit setting to deal with two zero-loss actions --- with bandit feedback the learner can never be sure what the entire loss vector is.

For Bandit-\hedgetree, we prove the following theorem.
Note that the bound enjoys the same $\order(\log(K))$ approximation ratio as in the full-information setting, but suffers $\order(\sqrt{T})$ additive regret.

\begin{theorem}
\label{theorem: bandit hat}
    {\rm Bandit}-\hedgetree (Algorithm~\ref{alg:exp3-tournament}) guarantees that for any ranking $\sigma$,
    \begin{align*}
       \E\left[\sum_{t=1}^T \ell_t(a_t) \right] 
        \leq \frac{(1+\lceil \log_2(K)\rceil) \cdot \frac{K\eta}{\mu}}{1-\e^{-\frac{K\eta}{\mu}}} \cdot \E\left[\sum_{t=1}^T \ell_t(\sigma(\calA_t))\right] 
         + \order\left(\frac{K N^2 }{\mu\left(1-\e^{-\frac{K\eta}{\mu}}\right)} + \mu T\right).
    \end{align*}
    Letting $\mu=\min\left\{N\sqrt{\frac{K}{T}}, 1\right\}$, $\eta=\frac{\mu}{K}$, the above is bounded by 
$
        \order(\log(K)) \cdot \E\left[\sum_{t=1}^T \ell_t(\sigma(\calA_t))\right] + \order\left(N\sqrt{KT} + KN^2\right). 
$
\end{theorem}

\begin{proof}
    By the same argument as in the proof of Lemma~\ref{lemma: algorithm 1 ineq 1}, there exists some $i$ such that $\widehat{a}_t^{i,z_t}=i$ and
\begin{align*}
    \one[\ell_t(a_t)=0] \cdot \ell_t(\widehat{a}_t) \leq \one[\ell_t(a_t)=0] \cdot \ell_t(i) \cdot \one[\{i,z_t\}\in U_t]. 
\end{align*}
Multiplying both sides by $\frac{|\calA_t| \cdot \one[\rho_t=1]}{\mu}$, we get
\begin{align*}
    \ell_t(\widehat{a_t}) \cdot \frac{|\calA_t|\cdot \one[\rho_t=1]\cdot\one[\ell_t(a_t)=0]}{\mu} \leq c_t^{i,z_t}(i). 
\end{align*}
 
Thus, 
\begin{align*}
    \ell_t(\widehat{a_t}) \cdot \frac{|\calA_t|\cdot \one[\rho_t=1]\cdot\one[\ell_t(a_t)=0]}{\mu} \leq \sum_{i<j}c_t^{i,j}(\widehat{a}_t^{i,j}). 
\end{align*}

By Lemma~\ref{lem: hedge bound approx} with $R=\frac{K}{\mu}$, we have

    \begin{align*}
    \scalebox{1.0}{$\displaystyle\E\left[\sum_{t=1}^T c_{t}^{i,j}(\widehat{a}_t^{i,j})\right]  \leq \frac{\frac{K\eta}{\mu}}{1-\e^{-\frac{K\eta}{\mu}}} \cdot \E\left[\min_\sigma \sum_{t=1}^T c_{t}^{i,j}(\sigma(i,j))\right] + \frac{(\ln 2)\frac{K}{\mu}}{1-\e^{-\frac{K\eta}{\mu}}}.  $}
\end{align*}
Then by the same argument as in the proof of Lemma~\ref{lemma: alorithm 1 ineq 3}, 
we have 
\begin{align*}
     c_{t}^{i,j}(\sigma(i,j)) \leq (1+\lceil \log_2(K) \rceil) \ell_t (\sigma(
     \calA_t)) \cdot \frac{|\calA_t|\cdot \one[\rho_t=1]\cdot\one[\ell_t(a_t)=0]}{\mu}. 
\end{align*}
Combining all of the above, we get 
\begin{align*}
\lefteqn{    \E\left[\ell_t(\widehat{a_t}) \cdot \frac{|\calA_t|\cdot \one[\rho_t=1]\cdot\one[\ell_t(a_t)=0]}{\mu} \right] }
\\    &\leq \frac{\frac{K\eta}{\mu}(1+\lceil \log_2(K) \rceil)}{1-\e^{-\frac{K\eta}{\mu}}} \cdot \E\left[\sum_{t=1}^T \ell_t(\sigma(
     \calA_t)) \cdot \frac{|\calA_t|\cdot \one[\rho_t=1]\cdot\one[\ell_t(a_t)=0]}{\mu}\right]  \\
     &\qquad \qquad + \order\left( \frac{N^2\frac{K}{\mu}}{1-\e^{-\frac{K\eta}{\mu}}}\right). 
\end{align*}
Taking the expectation over $a_t$ and $\rho_t$: 
\begin{align*} 
    \E\left[\sum_{t=1}^T\ell_t(\widehat{a_t}) \right] \leq  \; \frac{(1+\lceil \log_2(K)\rceil) \cdot \frac{K\eta}{\mu}}{1-\e^{-\frac{K\eta}{\mu}}} \cdot \E\left[\sum_{t=1}^T \ell_t(\sigma(\calA_t))\right]  + \order\left(\frac{K N^2 }{\mu\left(1-\e^{-\frac{K\eta}{\mu}}\right)} \right).
\end{align*}
Finally, using that $\E\left[\sum_{t=1}^T \one[\widehat{a}_t\neq a_t] \right] = \mu T$ completes the proof. 
\end{proof}


\paragraph{The \level Algorithm.}

\begin{algorithm}[t]
\caption{The \level algorithm}\label{alg:level}
\DontPrintSemicolon

\lForAll{actions $a\in[N]$}{let $\lev(a)\leftarrow 0$.}

\For{$t=1,\ldots, T$}{
  Let $a_t\in \argmin_{a\in \calA_t} \lev(a)$.
  
  Choose action $a_t$ and incur loss $\ell_t(a_t)$.
  
  \lIf{$\ell_t(a_t)=1$}{increment  $\lev(a_t)$ by $1$.}
}
\end{algorithm}

Finally, we consider the most challenging setup: bandit feedback without any restrictions on the number of zero-loss actions.
The algorithm we present is inspired by similar ideas of \cite{blum2018learning} for a very different problem, where a perfect ranking exists. This is generally not true in our setting, and our analysis is also new. The idea is to keep track of a \emph{level} for each action, and to always choose an action $a_t$ with the smallest level among all available actions in $\calA_t$. If the chosen action suffers a loss of $1$, then that action will be moved down by one level, i.e., its level increases by one (see Algorithm~\ref{alg:level}).
Note that this algorithm is deterministic, and we have the following deterministic guarantee:

\begin{theorem} \label{thm: level theorem}
The \level algorithm ensures:
$
        \sum_{t=1}^T \ell_t(a_t)\leq N\min_\sigma\sum_{t=1}^T \ell_t(\sigma(\calA_t)) + \frac{N(N-1)}{2}. 
$
\end{theorem}

The proof of this theorem makes use of the following key lemma. 

\begin{lemma} \label{lem: level lemma}
    Let $\lev_t(a)$ be the level of action $a$ at the beginning of round $t$. Then for every $t, a$, and $\sigma$,
$
        \lev_t(a) \leq m_\sigma(a) -1 + \sum_{\tau=1}^{t-1} \ell_\tau(\sigma(\calA_\tau)), 
$
    where $m_\sigma(a)$ is the rank of $a$ under $\sigma$. 
\end{lemma}

\begin{proof}
  We use induction on $t$.
  When $t=1$, the inequality clearly holds.
  Suppose that the following holds for all $a$: 
    \begin{align*}
        \lev_t(a)\leq m_\sigma(a) - 1 + \sum_{\tau=1}^{t-1}\ell_\tau(\sigma(\calA_\tau)). 
    \end{align*}
    We prove the bound for $t+1$.
    
    If the level of an action $a$ does not change at time $t$ (i.e., $\lev_{t+1}(a)=\lev_{t}(a)$), then the induction step is simple: 
    \begin{align*}
        \lev_{t+1}(a) 
        = \lev_t(a)\leq m_\sigma(a) - 1 + \sum_{\tau=1}^{t-1}\ell_\tau(\sigma(\calA_\tau))
        \leq m_\sigma(a) - 1  + \sum_{\tau=1}^{t}\ell_\tau(\sigma(\calA_\tau)). 
    \end{align*}
    
    Now consider an action $a$ with $\lev_{t+1}(a)\neq \lev_t(a)$. By our algorithm, this is only possible for $a=a_t$, and only when $\ell_t(a_t)=1$. Therefore, we only need to prove that $\lev_{t+1}(a_t)\leq m_\sigma(a_t) + \sum_{\tau=1}^t \ell_\tau(\sigma(\calA_\tau))$ under the assumption that $\ell_t(a_t)=1$. 
    First, if $\ell_t(\sigma(\calA_t))=1$, then

    \begin{align*}
        \lev_{t+1}(a_t)
        &=\lev_t(a_t)+1
        \leq m_\sigma(a_t) +  \sum_{\tau=1}^{t-1}\ell_\tau(\sigma(\calA_\tau)) 
        = m_\sigma(a_t) -1 +  \sum_{\tau=1}^{t}\ell_\tau(\sigma(\calA_\tau)).  
    \end{align*}
    Second, if $\ell_t(\sigma(\calA_t))=0$, then since $\ell_t(a_t)=1$, we have $\sigma(\calA_t)\neq a_t$. Because $\sigma(\calA_t)$ is the action that $\sigma$ ranks highest among $\calA_t$, we have $m_\sigma(\sigma(\calA_t)) \leq  m_\sigma(a_t)-1$. Therefore, 
\begin{align*}
 \lev_{t+1}(a_t) 
 &=\lev_{t}(a_t) + 1
 \; \stackrel{(*)}{\leq} \; \lev_t(\sigma(\calA_t)) + 1
 \; \leq \; m_\sigma(\sigma(\calA_t)) + \sum_{\tau=1}^{t-1}\ell_\tau(\sigma(\calA_\tau)) \\
 & \leq \; m_\sigma(a_t) - 1 + \sum_{\tau=1}^{t-1}\ell_\tau(\sigma(\calA_\tau))
 \; = \; m_\sigma(a_t) -1 + \sum_{\tau=1}^{t}\ell_\tau(\sigma(\calA_\tau)). 
\end{align*}
In the step marked $(*)$, we used the specific choice of $a_t$ made in the algorithm.
This finishes the induction.
\end{proof}

\begin{proof}[of Theorem~\ref{thm: level theorem}]
    Observe that the sum of $\lev_t(a)$ over $a$ is always the total number of mistakes the learner has made up to time $t-1$. Therefore, for any ranking $\sigma$,
    \begin{align*}
        \sum_{t=1}^T \ell_t(a_t) 
        &= \sum_{a\in[N]} \lev_{T+1}(a)
        \leq \sum_{a\in[N]} \left(m_\sigma(a) -1 + \sum_{t=1}^{T} \ell_t(\sigma(\calA_t))\right) \\
        &= \frac{N(N-1)}{2} + N\sum_{t=1}^{T} \ell_t(\sigma(\calA_t)),
    \end{align*}
    where the inequality is by Lemma~\ref{lem: level lemma}.
\end{proof}

With \level, we can actually deal with any sleeping expert/bandit problems with real-valued losses $\ell_{t}(a)\in[0,1]$. A reduction from the case of real-valued losses to binary losses can be done with random rounding: when facing a loss $\ell_t(a)$, the algorithm generates a randomized version $\ell_t'(a)$, which is $1$ with probability $\ell_t(a)$ and 0 otherwise; then $\ell_t'(\cdot)$ is fed to the \level algorithm as given above. This preserves the expectation of the losses suffered by the learner and any ranking (i.e., $\E[\ell_t'(a_t)]=\E[\ell_t(a_t)]$, $\E[\ell_t'(\sigma(\calA_t))]=\E[\ell_t(\sigma(\calA_t))]$ for any $t$ and any $\sigma$), and thus does not affect the expected regret. 

Note that while the \level algorithm can handle the most general case and enjoys $\order(N^2)$ additive regret, the approximation ratio is $N$, which could be much larger than $K$.

%% file: conclusions.tex

We revisited the problem of online learning with changing action sets in the adversarial setting and developed the first efficient algorithms with approximate regret guarantees, for both the general setting with bandit feedback and several special cases where significant improvements are obtained.
One clear open question is whether $\text{poly}(K)$ approximation ratio is achievable generally, without restrictions on the number of zero-loss actions, even for the full-information setting.
An intermediate step would be to show that for any constant
  number $z$ of zero-loss actions, there is an algorithm with regret approximation ratio
  $O(K^{f(z)})$ for some function $f$; we have so far only shown
  algorithms for $z \leq 2$.
Perhaps an even more basic question is whether there is a single algorithm that works when the number of zeros $z \in \{0, 1\}$ can change between rounds, and the algorithm does not know the number of zeros in a given round.
Another direction is to improve the additive $\order(\sqrt{T})$ regret for the bandit setting with one zero-loss action.